\documentclass[sigconf]{acmart}

\copyrightyear{2021}
\acmYear{2021}
\setcopyright{acmcopyright}\acmConference[CIKM '21]{Proceedings of the 30th ACM International Conference on Information and Knowledge Management}{November 1--5, 2021}{Virtual Event, QLD, Australia}
\acmBooktitle{Proceedings of the 30th ACM International Conference on Information and Knowledge Management (CIKM '21), November 1--5, 2021, Virtual Event, QLD, Australia}
\acmPrice{15.00}
\acmDOI{10.1145/3459637.3482246}
\acmISBN{978-1-4503-8446-9/21/11}

\usepackage{url}
\usepackage[utf8]{inputenc}
\usepackage{graphicx}
\usepackage{amsmath}
\usepackage{amsfonts}

\usepackage{booktabs}
\usepackage[normalem]{ulem}
\useunder{\uline}{\ul}{}
\usepackage{subcaption}
\usepackage{multirow}
\usepackage{algorithm}
\usepackage{algorithmic}
\usepackage{cleveref}
\usepackage{diagbox}
\newtheorem{theorem}{Theorem}[section]

\begin{document}
\fancyhead{}
\title{DCAP: Deep Cross Attentional Product Network for User Response Prediction}

\author{Zekai Chen}
\affiliation{%
  \institution{Department of Computer Science}
  \institution{George Washington University}
  \streetaddress{800 22nd NW St}
  \city{Washington, DC}
  \country{USA}
  \postcode{20052}
}
\email{zech\_chan@gwu.edu}

\author{Fangtian Zhong}
\affiliation{%
  \institution{Department of Computer Science}
  \institution{George Washington University}
  \streetaddress{800 22nd NW St}
  \city{Washington, DC}
  \country{USA}
  \postcode{20052}
}
\email{squareky\_zhong@gwu.edu}

\author{Zhumin Chen}
\affiliation{%
  \institution{School of Computer Science and Technology}
  \institution{Shandong University \country{China}}
}
\email{chenzhumin@sdu.edu.cn}

\author{Xiao Zhang}
\affiliation{%
  \institution{School of Computer Science and Technology}
  \institution{Shandong University \country{China}}
}
\email{xiaozhang@sdu.edu.cn}

\author{Robert Pless}
\affiliation{%
  \institution{Department of Computer Science}
  \institution{George Washington University}
  \streetaddress{800 22nd NW St}
  \city{Washington, DC}
  \country{USA}
  \postcode{20052}
}
\email{pless@gwu.edu}

\author{Xiuzhen Cheng}
\affiliation{%
  \institution{Department of Computer Science}
  \institution{George Washington University}
  \streetaddress{800 22nd NW St}
  \city{Washington, DC}
  \country{USA}
  \postcode{20052}
}
\email{cheng@gwu.edu}



\begin{abstract}
User response prediction, which aims to predict the probability that a user will provide a predefined positive response in a given context such as clicking on an ad or purchasing an item, is crucial to many industrial applications such as online advertising, recommender systems, and search ranking. For these tasks and many other machine learning tasks, an indispensable part of success is feature engineering, where cross features are a significant type of feature transformations. However, due to the high dimensionality and super sparsity of the data collected in these tasks, handcrafting cross features is inevitably time expensive. Prior studies in predicting user response leveraged the feature interactions by enhancing feature vectors with products of features to model second-order or high-order cross features, either explicitly or implicitly. However, these existing methods can be hindered by not learning sufficient cross features due to model architecture limitations or modeling all high-order feature interactions with equal weights. Different features should contribute differently to the prediction, and not all cross features are with the same prediction power. 

This work aims to fill this gap by proposing a novel architecture Deep Cross Attentional Product Network (DCAP), which keeps cross network's benefits in modeling high-order feature interactions explicitly at the vector-wise level. By computing the inner product or outer product between attentional feature embeddings and original input embeddings as each layer's output, we can model cross features with a higher degree of order as the network's depth increases. We concatenate all the outputs from each layer, which further helps the model capture much information on cross features of different orders. Beyond that, it can differentiate the importance of different cross features in each network layer inspired by the multi-head attention mechanism and Product Neural Network (PNN), allowing practitioners to perform a more in-depth analysis of user behaviors. Additionally, our proposed model can be easily implemented and train in parallel. We conduct comprehensive experiments on three real-world datasets. The results have robustly demonstrated that our proposed model DCAP achieves superior prediction performance compared with the state-of-the-art models. Public codes are available at https://github.com/zachstarkk/DCAP.
\end{abstract}




\maketitle

\section{Introduction}
\label{intro}
With the continuous and rapid growth of online service platforms, user response prediction (URP) has played an increasingly important role as the central problem of many online applications, such as online advertising \cite{Liu2015,Juan2016,Gai2017}, recommender systems \cite{Cheng2016,Karatzoglou2017}, and web search \cite{Shan2016}. In online advertising, quantifying user intent allows advertisers or social platforms to target ads' right users. It leads to the judicious use of multi-billion marketing dollars and also renders a pleasant user experience. In recommender systems, correctly predicting the \textit{rating} or \textit{preference} a user would respond to an item can also create a delightful user experience while driving incremental revenue. 

As illustrated in \figurename\ \ref{fig:user response prediction}, based on the user level features collected within a particular historical window of any specific task, the predictive system can estimate in advance how likely a user will provide a predefined positive response, e.g., clicks an ad (also known as click-through rate, CTR), likes a post, purchases an item, etc., in a given context \cite{Menon2011}.

As \figurename\ \ref{fig:user response prediction} also reflects, user-level data collected from these tasks are usually in a multi-field categorical form, for example, $\left[\texttt{Country=\textit{USA}, Gender=\textit{Male}},\cdots\right]$ which is normally transformed into binary features via one-hot encoding. When performing machine learning on solving such prediction tasks, an indispensable part of success is applying feature interactions across raw features, which has been emphasized in related literature \cite{Blondel2016,Zhang2016,Cheng2016,Guo2017}. For instance, young male users may have a higher chance of clicking ads toward video games, which indicates that the \textit{cross feature} regarding \textit{age} and \textit{gender} can be a significant predictor for estimating whether users will click on some certain types of ads. However, the binarized user data leads to a high-dimensional and sparse feature space (e.g., the well-known CTR prediction dataset Criteo\footnote{http://labs.criteo.com/2014/02/kaggle-display-advertising-challenge-dataset/} has a feature dimension of over one million with sparsity over 99\% \cite{Song2019}) where handcrafting powerful cross features are inevitably time expensive, and results can rarely be generalized to unseen high-order feature interactions \cite{Cheng2019}. 

Some pioneering work was then proposed to overcome the limitation by leveraging feature interactions in an automation fashion. For example, Factorization Machines (FMs) \cite{Rendle2010} was proposed to model second-order cross features by parameterizing a cross feature's weight as the constituent features' inner product of the embedding vectors. Other recent developments \cite{Zhang2016,Guo2017,He2017a,Lian2018} have also augmented FMs with deep neural networks (DNNs) to model more expressive high-order feature interactions, which have gained promising results. However, we argue that two remaining challenges can potentially hinder these approaches. First, for models relying on DNNs to achieve the capability of modeling high-order cross features (e.g., DeepFM \cite{Guo2017}, NFM \cite{He2017a}, etc.), an obvious drawback is that deep neural networks only learn feature interactions in an implicit way, which have been shown inefficient in learning multiplicative feature interactions \cite{Beutel2018,Song2019}. It also significantly reduces the model's interpretability. Second, for models able to explicitly model high-order feature interactions (e.g., xDeepFM \cite{Lian2018}, etc.), they are limited by assigning equal weights to factorized embeddings. In most real-business scenarios, it is undeniable that different predictors might have different predicting power, and not all the features contain useful signals for making the forecasting. Note that Attentional Factorization Machine (AFM) \cite{He2017} produces an attention-based pooling layer to differentiate the cross feature importance such that the influence of less useful feature interactions can be compromised by assigning lower weights. Unfortunately, AFM only aims to model second-order feature interactions, which lacks the capability of capturing high-order feature information. 

\begin{figure}
    \centering
    \includegraphics[width=\linewidth]{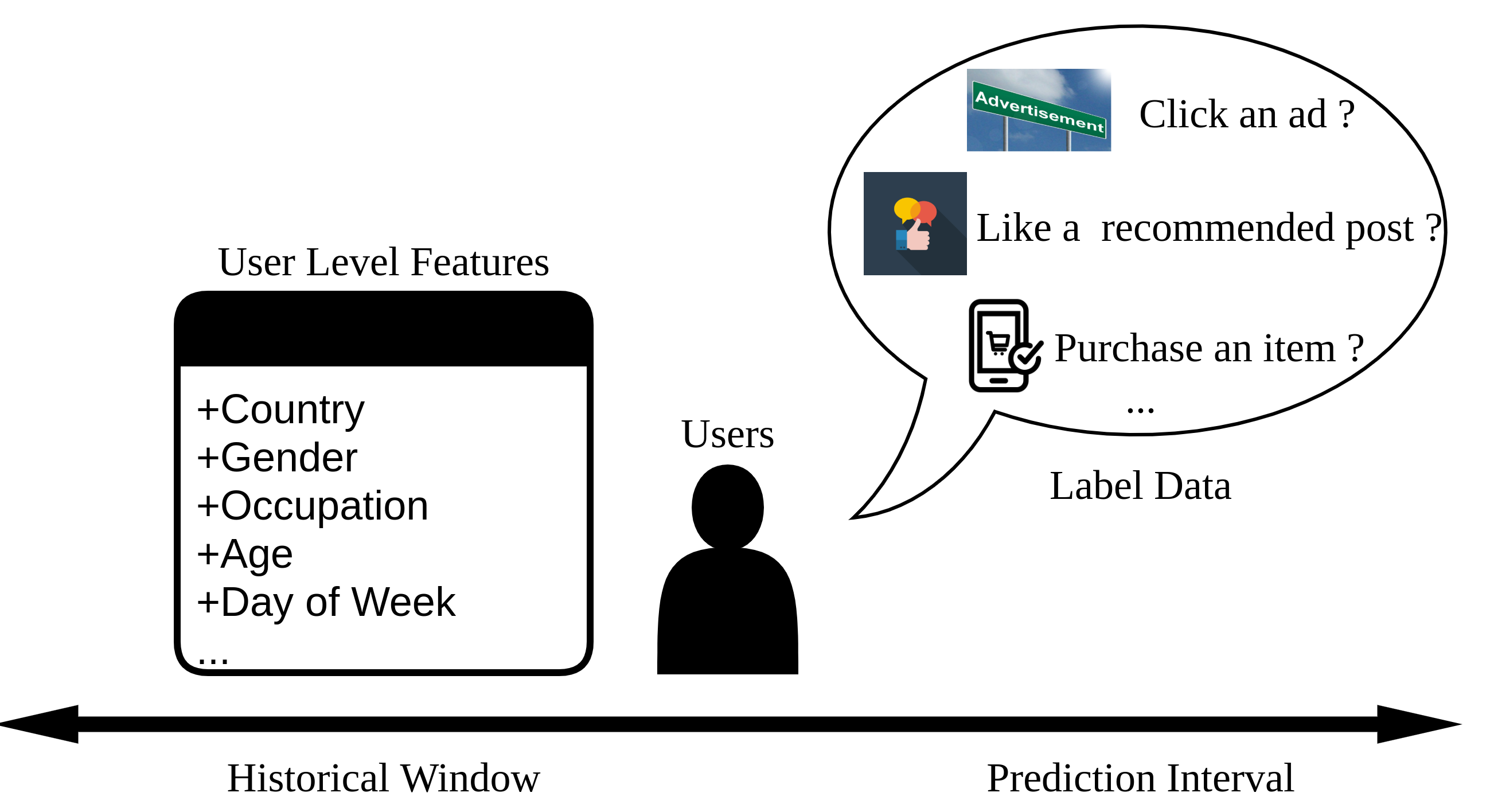}
    \caption{An visualization of illustrating user response prediction tasks.}
    \label{fig:user response prediction}
    \vspace{-0.5cm}
\end{figure}

To overcome these challenges, we propose a novel architecture \textit{\textbf{D}\textit{eep}} \textit{\textbf{C}ross} \textit{\textbf{A}ttentional} \textit{\textbf{P}roduct} network (DCAP) that aims to explicitly learn high-order cross features with discriminated importance at a vector-wise level. Our approach is inspired by the $\textit{cross network}$ \cite{Wang2017} architecture and utilizes the multi-head self-attention mechanism proposed in Transformer \cite{Vaswani2017} which is now the $\textit{defacto}$ state-of-the-art model in broad NLP domains \cite{Devlin2019,Brown2020}. To be specific, we first form the attentional feature embeddings in the first network layer using the self-attention mechanism. We then compute the inner product or outer product between attentional feature embeddings and original input embeddings as each layer's output. By increasing the network's depth, we further generate the cross attentional features across each layer using the last layer's outputs and the original inputs. As the depth of the network increases, the degree of feature interactions grows up correspondingly. Significantly, this attention mechanism greatly enhances our model's interpretability and transparency, allowing practitioners to perform a more in-depth analysis of user behaviors. Inspired by PNN \cite{Qu2017}, for each layer, we represent all cross features in product fashion. To summarize, we make the following contributions:

\begin{itemize}
    \item We propose a novel Deep Cross Attentional Product Network (DCAP) that explicitly models high-order feature interactions with discriminated importance at a vector-wise level in an automation fashion. It can efficiently capture highly nonlinear feature interactions of bounded degree concerning the network's depth, requiring no handcrafted feature engineering or large-scale searching at a tremendous computational cost. 
    \item We concatenate outputs from each layer such that the network outputs consist of all cross features with the degree of order ranging from 1 to the highest, which contains richer predictor information.
    \item We conduct comprehensive experiments on three real-world datasets, and the results have demonstrated that our proposed DCAP achieves superior performance compared with the state-of-the-art models.
\end{itemize}

The remainder of this paper is organized as follows. The related work is first reviewed in Section \ref{related_work}. We further provide some preliminary knowledge in Section \ref{methods}, which is necessary for understanding our model architecture and the inspirations of our design. We then describe our proposed architecture DCAP in Section \ref{methods}. We conduct experiments and report the results in Section \ref{experiments}. For simplicity, we take the click-through rate (CTR) estimation in online advertising as working examples to explore our model's potential effectiveness as it is one of the most critical applications of user response prediction. Finally, we conclude our work in Section \ref{conclusion}. 

\section{Related Work}
\label{related_work}
The problem of user response prediction in machine learning is essentially a binary classification problem with prediction likelihood as the training objective \cite{Menon2011,Qu2017}. It is widely acknowledged that modeling feature interactions are crucial for a good performance \cite{Cheng2016,Guo2017,Wang2017,Lian2018}.
\subsection{Individual Models} 
Rendle \textit{et al.} \cite{Rendle2010} proposed the pioneering work Factorization Machines (FMs) to capture at most second-order feature interactions. It has been demonstrated effective in both recommender systems \cite{He2017,Pan2018} and CTR prediction \cite{Guo2017,Lian2018}. With the success of FMs, different variants of FMs have been proposed. For example, Juan \textit{et al.} \cite{Juan2016} proposed Field-aware Factorization Machine (FFM) to model fine-grained interactions across features within different fields. Cheng \textit{et al.} \cite{Cheng2014} proposed GBFM and He \textit{et al.} \cite{He2017} introduced AFM to capture the importance of different second-order feature interactions. However, all these approaches focused on modeling low-order feature interactions, which are not sufficient for the model to capture high-order information. Blondel \textit{et al.} \cite{Blondel2016} proposed a high-order FMs algorithm (HOFMs) to generalize the capability of FMs for modeling any cross feature at any degree of order. However, it is limited by the polynomial model complexity.

\subsection{Ensemble Models}
As deep neural networks (DNNs) demonstrated the capability to explore high-order hidden patterns, more ensemble models were then proposed to integrate with DNNs as encounter parts to boost the prediction performance. Cheng \textit{et al.} \cite{Cheng2016} developed Wide \& Deep for online recommendation at Google combining linear model and DNN, which achieves a remarkable performance in APP recommendation. He \textit{et al.} \cite{He2017a} introduced NFM, Guo \textit{et al.} \cite{Guo2017} proposed DeepFM, and Zhang \textit{et al.} \cite{Zhang2019} proposed FNFM to further combine FMs with DNNs to gain the capability of modeling implicit high-order feature interactions. Qu \textit{et al.} \cite{Qu2017} proposed Product Neural Network (PNN) that utilizes the inner product and outer product of features to model the inter-field feature interactions followed by DNNs. Wang \textit{et al.} \cite{ Wang2017} devised a $\textit{cross network}$ (DCN) architecture that jointly learn both explicit and implicit high-order feature interactions in an automation fashion. It can learn feature interactions efficiently; however, all interactions come in a bit-wise fashion, bringing more implicitness. More importantly, it is inefficient in learning multiplicative feature interactions \cite{Beutel2018}. Lian \textit{et al.} \cite{Lian2018} combined \textit{compressed interaction network} with a multi-layer perceptron (xDeepFM) to learn high-order feature interactions explicitly at a vector-wise level. Additionally, Cheng \textit{et al.}  \cite{Cheng2019} devised AFN with a logarithmic transformation layer that converts each feature's power in a feature combination into the coefficient to be learned. Despite the promising results, these methods can be hindered due to the low interpretability caused by DNNs' implicitness or assigning equal weights to all the feature interactions without discrimination. 

Moreover, AutoFIS \cite{Liu2020} can automatically identify important feature interactions for factorization models with computational cost just equivalent to training the target model to convergence. Song \textit{et al.} \cite{Song2019} adopted the multi-head attention mechanism (MHA) \cite{Vaswani2017} for modeling feature interactions due to its superiority over modeling attentional pairwise feature correlations. Feng \textit{et al.} \cite{Feng2019} also directly applied this self-attention mechanism in the Deep Session Interest Network (DSIN) for CTR prediction. However, these models' performance can be limited in modeling high-order complicated feature interactions by merely stacking the self-attention blocks. Thus, we propose a model that can fully utilize the MHA mechanism to model explicit high-order feature interactions with discriminated importance.

\section{Methodologies}
\label{methods}
Inspired by the Deep\&Cross Network \cite{Shan2016}, we design this Deep Cross Attentional Product Network (DCAP) with the following considerations: (1) high-order explicit feature interactions at vector-wise level; (2) differentiating the significance of feature interactions with any degree of order; (3) easy parallelization for computation efficiency. This section presents the architecture of our proposed model and analyzes both time and space complexity. We present a comprehensive description of how to learn a low-dimensional dense feature embedding that models high-order cross features with discriminated importance. 

\subsection{Sparse Input and Embedding Layer}
\label{embedding layer}
In web-scale recommender systems or CTR prediction problems, the input features are usually sparse with tremendous dimension and present no clear spatial or temporal correlation. One common approach is transforming the multi-field categorical data into a high-dimensional sparse feature space via field-aware one-hot encoding. An embedding layer is developed upon the raw feature input to compress it to a low dimensional, dense real-value vector representation through a shared latent space. Expressly, we represent each categorical feature with a low-dimensional vector as following:
\begin{equation}
    \mathbf{x}_{i}=\mathbf{V}_{i}\mathbf{e}_{i}
\end{equation}
where $\mathbf{V}_{i}$ is an embedding matrix for field $i$, and $\mathbf{e}_{i}$ is a one-hot vector. If the field is multivalent, we take the field embedding as the sum of feature embedding. The output of the embedding layer is then a wide concatenated vector as:
\begin{equation}
    \mathbf{x} = \left[\mathbf{x}_{1}, \cdots, \mathbf{x}_{n}\right]
\end{equation}
where $n$ denotes the number of feature fields, and $\mathbf{x}_{i}\in\mathcal{R}^{d}$ denotes the embedding of one field.
The embedding layer is illustrated in \figurename\ \ref{fig:embedding layer}. 
\begin{figure}
    \centering
    \includegraphics[width=\linewidth]{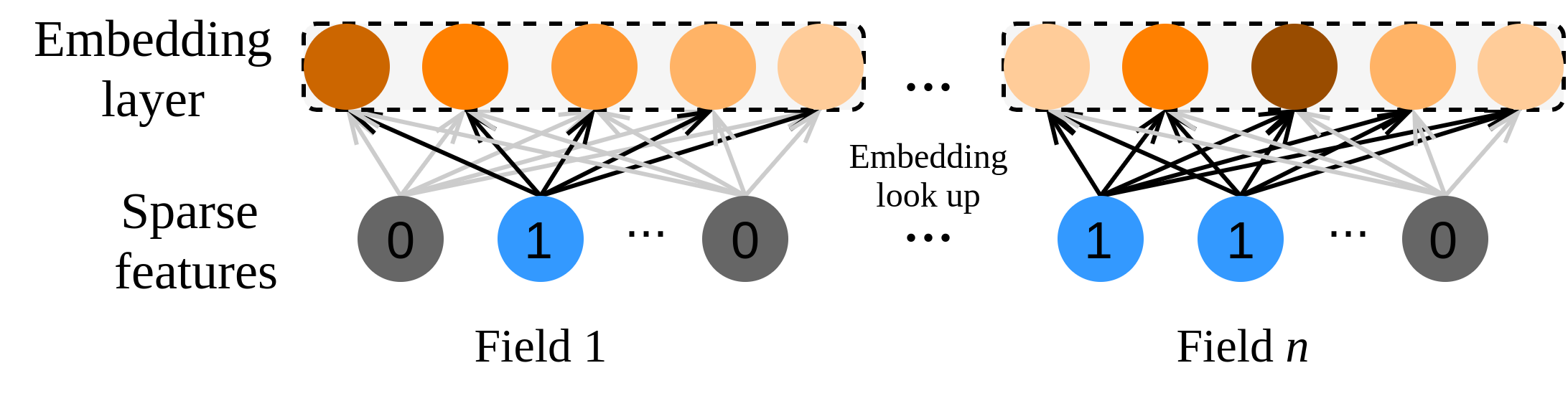}
    \vspace{-0.2cm}
    \caption{An example of the embedding layer with an embedding dimension of 5.}
    \label{fig:embedding layer}
    \vspace{-0.6cm}
\end{figure}
We apply the input and embedding layer as the first stage of our model, which converts a set of sparse representations of input features into dense vectors as shown in the left panel of \figurename\ \ref{fig:dcap}.

\subsection{Multi-head Self-attention}
\subsubsection{Motivations}
One efficient way to enable feature interactions to contribute differently to the prediction is the attention mechanism. Attention-based models can be traced back to the Neural Machine Translation \cite{Bahdanau2015,Luong2015}, where Recurrent Neural Networks (RNNs) applied the attention mechanism between encoder and decoder. However, the recurrent mechanism for temporal dependency modeling is not easy to parallelize. Transformer \cite{Vaswani2017} extends the attention mechanism to a multi-head self-attention with scaled dot-production. It has been widely used in diverse NLP domains and many other tasks such as recommendation system \cite{Feng2019}, information retrieval \cite{Tay2018}, and computer vision \cite{Ramachandran2019}. The main idea of self-attention is a soft addressing process within pair-wise token representations in a sequence. Only a subset of tokens are worth more \textit{attention} instead of assigning equal weights to all the features. For example, consider a sentiment classification problem with the texts “I feel so happy to buy a new laptop.” The words besides “happy” are not indicative of the positive emotion exhibited in this sentence. Those interactions involving irrelevant features can be considered as noises that have less contribution to the prediction. In real-world applications, different predictor variables usually have different predictive power, and not all features contain a useful signal for estimating the target. The interactions with less useful features should be assigned a lower weight as they contribute less to the prediction. Here we extend this multi-head attention mechanism to model the dependencies across different feature fields due to its capability of differentiating the importance of feature interactions.
\subsubsection{Formulations}
Specifically, for a set of input feature embeddings $\mathbf{X} = (\mathbf{x}_1, \mathbf{x}_2, \cdots, \mathbf{x}_n)^{T}\in \mathcal{R}^{n\times d}$ (with number of feature fields $n$ and dimensionality $d$), where $\mathbf{x}_i\in \mathcal{R}^{d}$. The self-attention function firstly projects them into queries $\mathbf{Q}\in \mathcal{R}^{n\times d_{q}}$, keys $\mathbf{K}\in \mathcal{R}^{n\times d_{k}}$ and values $\mathbf{V}\in \mathcal{R}^{n\times d_{v}}$, with different, learned linear projections to $d_{q}$, $d_{k}$ and $d_{v}$ dimensions, respectively. Then a particular scaled dot-product attention was computed to obtain the weights on the values as:
\begin{equation}
    \textsc{Attention}(\mathbf{Q}, \mathbf{K}, \mathbf{V})=\textsc{softmax}(\frac{\mathbf{Q}\mathbf{K}^{T}}{\sqrt{d_{k}}})\mathbf{V}
\end{equation}

For any single head, this attention mechanism operates on the input embedding and computes a new embedding output $\mathbf{Z}=(\mathbf{z}_1, \mathbf{z}_2, \cdots, \mathbf{z}_n)^{T}$ of the same number of fields where $\mathbf{z}_i\in \mathcal{R}^{d_v}$. 

Each output element, $\mathbf{z}_i$, is computed as weighted sum of a linearly transformed input elements:
\begin{equation}
\mathbf{z}_{i}=\sum_{j=1}^{n} \alpha_{i j}\left(\mathbf{x}_{j} W^{V}\right)
\end{equation}

Each weight coefficient, $\alpha_{ij}$, is computed using a softmax funtion:
\begin{equation}
\alpha_{i j}=\frac{\exp e_{i j}}{\sum_{k=1}^{n} \exp e_{i k}}
\end{equation}

And $e_{ij}$ is computed by the attention function that essentially captures the correlations between queries and keys using this dot-product so as to perform a soft-addressing process:
\begin{equation}e_{i j}=\frac{\left(\mathbf{x}_{i} W^{Q}\right)\left(\mathbf{x}_{j} W^{K}\right)^{T}}{\sqrt{d_{v}}}\end{equation}
where $W^Q\in \mathcal{R}^{d\times d_q}$, $W^K\in \mathcal{R}^{d\times d_k}$, $W^V\in \mathcal{R}^{d\times d_v}$ are parameter matrices.

Multi-head attention allows the model to jointly attend to information from different representation subspaces of different fields. The weighted sum of all concatenated heads is then computed as following:
\begin{equation}
    \textsc{MultiHead}(\mathbf{Q}, \mathbf{K}, \mathbf{V})=\textsc{Concat}(\textit{head}_{1}, \cdots, \textit{head}_{h})W^{O}
\end{equation}
in which, $h$ is the number of total heads.
Each head is defined as:
\begin{equation}
    \textit{head}_{i}=\textsc{Attention}(\mathbf{X}W_{i}^{Q}, \mathbf{X}W_{i}^{K}, \mathbf{X}W_{i}^{V})
\end{equation}
where the projections are parameter matrices $W_{i}^{Q}\in \mathcal{R}^{d\times d_{q}}$, $W_{i}^{K}\in \mathcal{R}^{d\times d_{k}}$, $W_{i}^{V}\in \mathcal{R}^{d\times d_{v}}$ and $W^{O}\in \mathcal{R}^{hd_{v}\times d}$. In practice, we usually set $d_{q}=d_{k}=d_{v}$, and $d=hd_v$.

\subsection{Deep Cross Attentional Product Network}
\label{dcap}
\begin{figure}
    \centering
    \includegraphics[width=\linewidth]{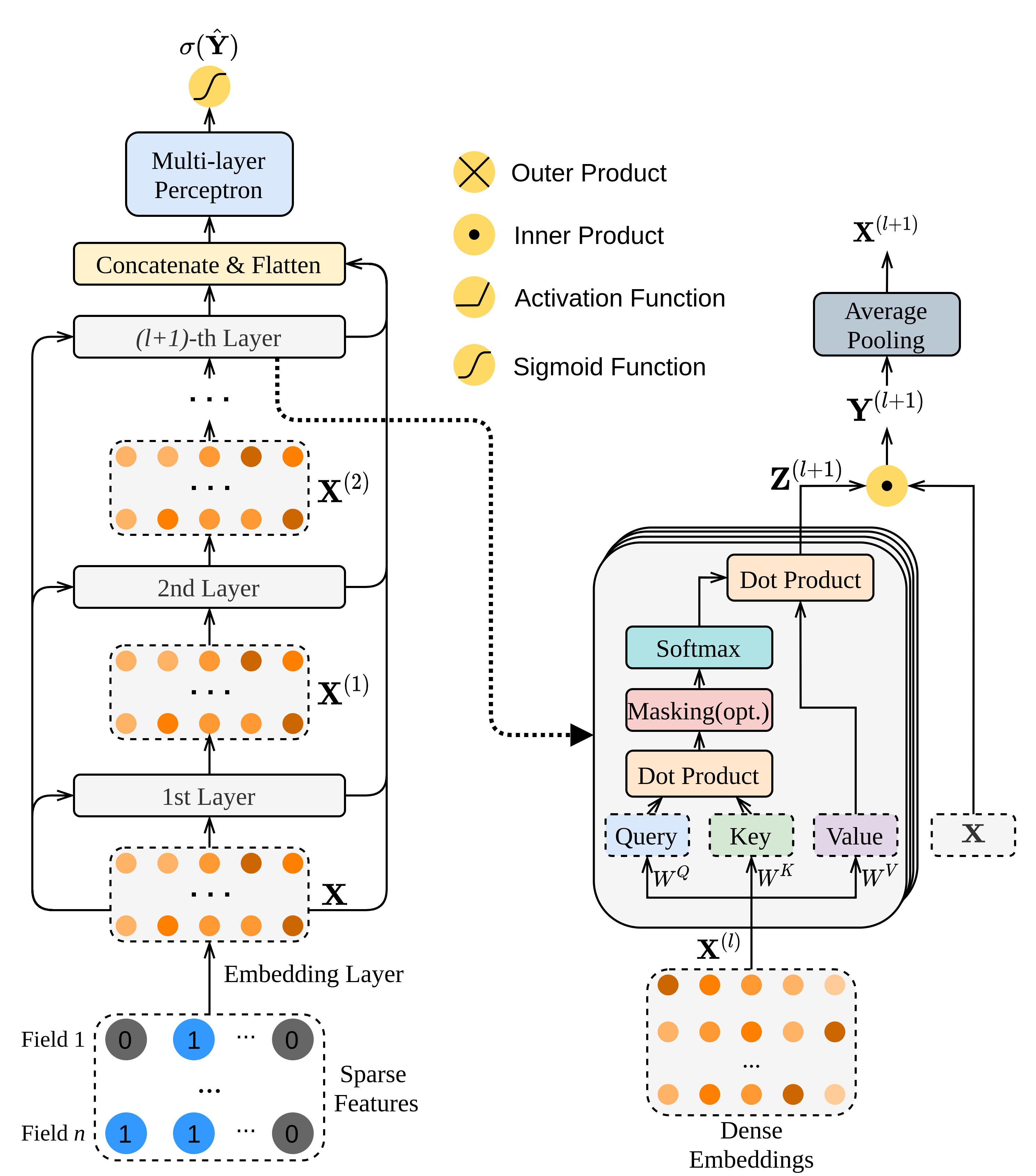}
    \vspace{-0.1cm}
    \caption{The overall architecture of our Deep Cross Attentional Network consists of multiple cross attentional product layers (left panel). Specifically, we illustrate the modeling process in detail for the $(l+1)$-th layer in the right panel.}
    \label{fig:dcap}
    \vspace{-0.4cm}
\end{figure}

In order to explore richer explicit high-order feature interactions by deepening the attention network, we design cross attention layers with each layer having the following functions:
\begin{align} 
 \mathbf{Z}^{(l+1)} &= \textsc{MultiHead}(\mathbf{Q}^{(l)}, \mathbf{K}^{(l)}, \mathbf{V}^{(l)}) \\
 \mathbf{P}^{(l+1)} &= \mathbf{Z}^{(l+1)}\odot|\otimes \mathbf{X} \\
 \mathbf{Y}^{(l+1)} &= \sum_{d} \mathbf{P}^{(l+1)} \\
 \mathbf{X}^{(l+1)} &= \textsc{AvgPooling}(\mathbf{P}^{(l+1)})
\end{align}
where $\mathbf{Q}^{(l)}$, $\mathbf{K}^{(l)}$ and $\mathbf{V}^{(l)}$ are projections of $\mathbf{X}^{(l)}\in \mathcal{R}^{n\times d}$ that are embedding vectors denoting the inputs for the $(l+1)$-th layer. $\mathbf{P}^{(l+1)}\in\mathcal{R}^{n(n-1)/2\times d}$ represents the production between the weighted outputs of multi-head attention $\mathbf{Z}^{(l)}\in \mathcal{R}^{n\times d}$ and the original input embeddings $\mathbf{X}\in \mathcal{R}^{n\times d}$ using either inner product $\odot$ or outer product $\otimes$. We further sum up the products along the embedding dimension to get $\mathbf{Y}^{(l+1)}\in\mathcal{R}^{n(n-1)/2}$ as the output (used for prediction) of the $(l+1)$-th layer. The average pooling layer then involves calculating the average for each patch of the feature interactions which means that each $k$ (kernal size) cross features is down sampled to the average value of them. In output layer, we concatenate the flattened input embeddings with outputs from each layer followed by a fully connected neural network. A visualization of the general architecture is shown in \figurename\ \ref{fig:dcap}.

\subsubsection{Product Operation}
We adopt the product operation to model the feature interactions as inspired by \cite{Qu2017}. For inner product, we define the operation of two vectors $\mathbf{a}\in\mathcal{R}^{d}$ and $\mathbf{b}\in\mathcal{R}^{d}$ as following:
\begin{equation}
\mathbf{a} \odot \mathbf{b} \triangleq (\mathbf{a}_{1} \mathbf{b}_{1}, \cdots, \mathbf{a}_{i} \mathbf{b}_{i}, \cdots, \mathbf{a}_{n} \mathbf{b}_{n} )
\end{equation}
For outer product, the definition of the operation is as following:
\begin{equation}
\mathbf{a} \otimes \mathbf{b} \triangleq (\sum_{j} \mathbf{a}_{1} \mathbf{b}_{j}, \cdots, \sum_{j} \mathbf{a}_{i} \mathbf{b}_{j}, \cdots, \sum_{j} \mathbf{a}_{n} \mathbf{b}_{j})
\end{equation}
Now, if we apply inner product operation between $\mathbf{Z}^{(l+1)}$ and $\mathbf{X}$, then
\begin{equation}
    \mathbf{Z}^{(l+1)}\odot\mathbf{X}=\{\mathbf{z}_{i}^{(l+1)}\odot\mathbf{x}_{j}\}_{(i, j)\in\mathcal{R}_{x}}
\label{eq16}
\end{equation}
where $\mathcal{R}_{x}=\{(i, j)\}_{1\leq i, j\leq n, j > i}$. Note that we only use the cross relationships between features from different fields once such that we expand $n$ features to $n(n-1)/2$ interacted cross features for each layer.

\begin{figure}
    \centering
    \vspace{-0.4cm}
    \includegraphics[width=\linewidth]{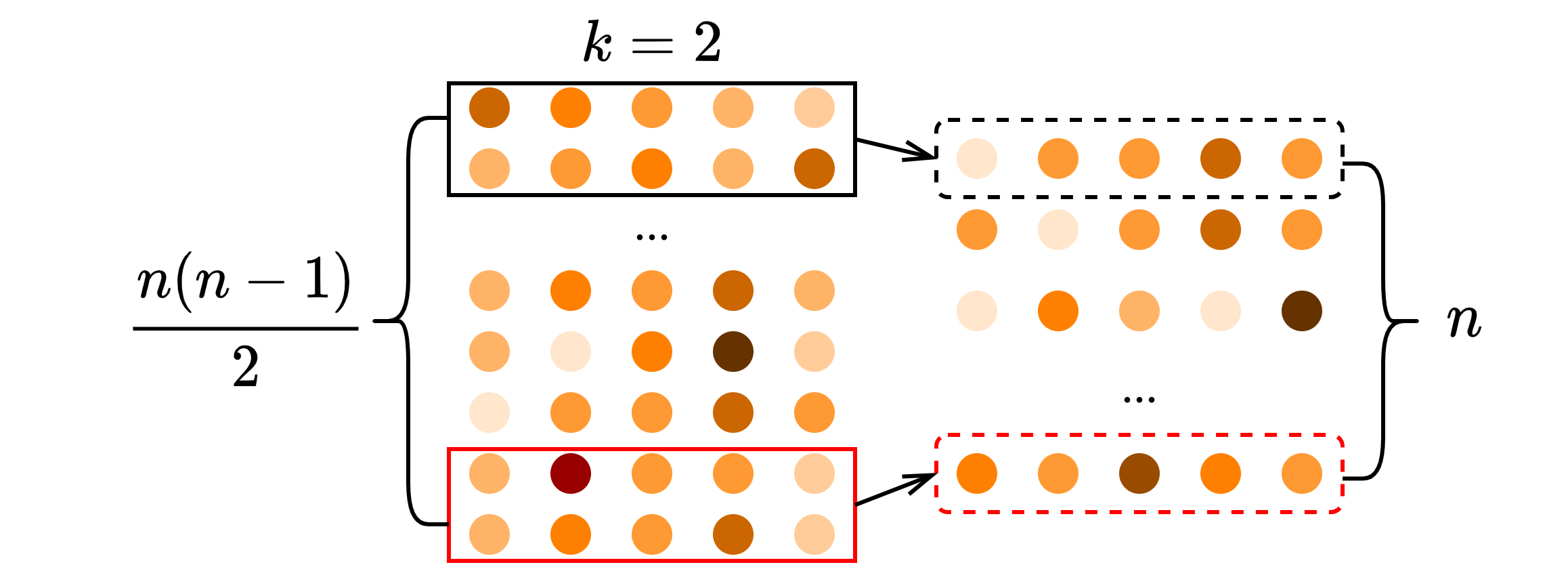}
    \caption{An example of the average pooling process with a kernel size as 2.}
    \label{fig:pooling}
    \vspace{-0.4cm}
\end{figure}

\subsubsection{Adaptive Average Pooling}
In fact, the production result $\mathbf{P}^{(l+1)}$ of the $(l+1)$-th layer has a total number of $\mathcal{O}(n^2)$ cross feature representations, and it will go quadratically as we further pass it to the multi-head attention in the next layer since the dot-product attention computation takes another $\mathcal{O}(n^2)$ complexity, $n$ is the number of input features. Average pooling is a common technique used in computer vision, especially in the pooling layer of convolutional neural network \cite{Lin2014}. In our model, we extend the average pooling approach to avoid the quadratic complexity caused by the depth of network without losing too much information. The general 1D average pooling process is as following:
\begin{equation}
\textit{Output}=\frac{1}{{k}} \sum_{{m}=0}^{{k}-1} \textit{Input}\left(\textsc{Stride} \times {l}+{m}, \cdots \right)
\end{equation}
where $\textit{output}\in \mathcal{R}^{L_{\textit{out}}\times d}$ and $\textit{input}\in \mathcal{R}^{L_{\textit{in}}\times d}$. 
Adaptive average pooling strategy can adaptively tune the kernel size $k$ and stride steps to our desired cross feature dimension since one can easily conclude that 
${L}_{\textit{out }}=\left\lfloor\frac{{L}_{\textit{in }}+2 \times \textsc{Padding }-\textsc{Kernel Size}}{\textsc{Stride }}+1\right\rfloor$. \figurename\ \ref{fig:pooling} is a visualization of the average pooling process. 

By applying the average pooling method, we compress down the cross features space from $\mathcal{O}(n^2)$ to $\mathcal{O}(n)$ while each new feature is a combination of a subset of original cross features. 

\subsection{Combination with Bit-wise Feature Interactions}
Upon the deep feature interactions, we concatenate all the outputs from each layer with the flattened input embedding followed by a multi-layer perceptron (also known as a fully connected neural network with multiple layers) to further explore the implicit bit-wise interactions. Suppose a cross attentional product network with $L$ layers, the final output would be computed as:
\begin{equation}
    \widehat{\mathbf{Y}} = \textsc{Mlp}(\textsc{Concat}[\mathbf{X}_{\text{flatten}}, \mathbf{Y}^{(1)}, \cdots, \mathbf{Y}^{(L)}])
\end{equation}
where the architecture of $\textsc{Mlp}$ can be referred to Section \ref{bit-wise interactions}. At the end, the output of the hidden layers is transformed to the final user clicking probabilities through sigmoid function $\sigma(\widehat{\mathbf{Y}}) = 1/(1+e^{-\widehat{\mathbf{Y}}})$. 

For binary classifications, the loss function is the log loss:
\begin{equation}
\mathcal{L}=-\frac{1}{N} \sum_{i=1}^{N} \mathbf{y}_{i} \log \hat{\mathbf{y}}_{i}+\left(1-\mathbf{y}_{i}\right) \log \left(1-\hat{\mathbf{y}}_{i}\right)
\end{equation}
where $N$ is the total number of training instances. The optimization process is to minimize the following objective function:
\begin{equation}
\mathcal{J}=\mathcal{L}+\lambda_{*}\|\Theta\|
\end{equation}
where $\lambda_{*}$ denotes the regularization term and $\|\Theta\|$ denotes the set of parameters, including these in both cross feature modeling layers and DNN part. 

\subsection{Model Analysis}
In this section, we analyze the proposed DCAP to study the model complexity and potential effectiveness. 
\subsubsection{Model Effectiveness}
First of all, we argue that this deep and cross structure is able to model high-order explicit feature interactions as the degree of cross features grows with layer depth.
\begin{theorem}
Consider an $l$-layer deep cross attentional product network with the $(m+1)$-th layer defined as Section \ref{dcap}. Let the input embedding be $\mathbf{X}$, the output of attentional feature embeddings in $(m+1)$-th layer be $\mathbf{Z}^{(m+1)}$. Then, the highest degree of cross features captured in terms of original input embedding $\mathbf{X}$ is $l+1$.
\end{theorem}

\begin{proof}
As an example, for the first layer, as Eq. \ref{eq16} demonstrates,
\begin{equation}
\begin{split}
    \mathbf{z}_{i}^{(1)}\odot\mathbf{x}_{k} &= \sum_{j=1}^{n} \alpha_{i j}\left(\mathbf{x}^{(l)}_{j} W^{V}\right)\odot\mathbf{x}_{k} \\
    &= \sum_{j=1}^{n} \alpha_{i j}\left(\mathbf{x}^{(l)}_{j} W^{V}\odot\mathbf{x}_{k}\right) \\
    &= \mathbf{p}^{T}\sum_{j=1}^{n} \alpha_{i j}\left(\mathbf{x}^{(l)}_{j} \odot\mathbf{x}_{k}\right)
\end{split}
\end{equation}
which models cross features ranging from $\mathbf{x}_{1}\mathbf{x}_{k}$ to $\mathbf{x}_{n}\mathbf{x}_{k}$ and the degree of cross feature in terms of input embedding $\mathbf{x}$ is 2. Note that the average pooling process does not change the order of feature interactions since it can be viewed as multiplying an additional averaging coefficient $\beta_{ij}$ to the cross feature representations. Suppose this statement holds for $l=m$. For $(m+1)$-th layer, we have:
\begin{equation}
\begin{split}
    \mathbf{z}_{i}^{(m+1)}\odot\mathbf{x}_{k} 
    &= \mathbf{p}^{T}\sum_{j=1}^{n} \alpha_{i j}\left(\mathbf{x}^{(m)}_{j} \odot\mathbf{x}_{k}\right)
\end{split}
\end{equation}
where $\mathbf{x}^{(m)}_{j} \odot\mathbf{x}_{k}$ can model cross features with degree of order $(m+1)$ since $\mathbf{x}^{(m)}_{j}$ represents the $m$-th order cross features. By induction hypothesis, one can easily prove that the highest degree in terms of original input embedding $\mathbf{X}$ for an $l$-layer cross attentional product network is $l+1$.
\end{proof}

\subsubsection{Model Complexity}
For computing multi-head attention, the total number of Mult-Adds is  $\mathcal{O}(2n^2d+4nd^2)$. For production of features, it takes another $\mathcal{O}(nd(n-1)/2)=\mathcal{O}(n^2d)$ Mult-Adds. So the total computation complexity for each layer in terms of Mult-Adds is $\mathcal{O}(3n^2d+4nd^2)$. A cross attentional product network with $l$ layers would take a total number of Mult-Adds as $\mathcal{O}(3n^2dl+4nd^2l)$. For the feed-forward network, the input size is $d'=nd+ln(n-1)/2$ such that the total number of Mult-Adds is $\mathcal{O}(d'h_1+h_1h_2+h_2)$, where $h_1, h_2$ are the hidden size of each hidden layer, respectively. For space complexity, the attention part takes $\mathcal{O}(4d^2)$ including three projection matrices and one final weighted sum layer. For DNN part, it contains weight matrices that are $\mathcal{O}(d'h_1+h_1h_2+h_2)$. For a model with $l$ layers, the total space complexity would be $\mathcal{O}(4d^2l+d'h_1+h_1h_2+h_2)$, which is quadratic with respect to both number of fields $n$ and embedding dimension $d$. 


\section{Experiments}
\label{experiments}
In this section, we conduct extensive experiments to answer the following questions:
\begin{itemize}
    \item (\textbf{RQ1}) How does our proposed Deep Cross Attentional Product Network (DCAP) perform in high-order feature interactions learning compared to the other state-of-the-art methods?
    \item (\textbf{RQ2}) How do the different hyper-parameter settings affect model performance?
    \item (\textbf{RQ3}) Is the multi-head attention mechanism vital for our model, and how can it essentially be useful in prediction?
\end{itemize}

\subsection{Experimental Setup}
\subsubsection{Datasets}
As the click-through rate (CTR) prediction in online advertising is one of the most critical user response domains, we take it as our experimental examples to explore our model's potential effectiveness. We evaluate the proposed model on the following three public datasets: 1) \textbf{Criteo}\footnote{http://labs.criteo.com/2014/02/kaggle-display-advertising-challenge-dataset/} Display Ads dataset is a famous benchmark for ads click-through rate prediction. It has 45 million users' clicking records on displayed ads. It contains 13 numerical feature fields and 26 categorical feature fields where each category has a high cardinality. 2) \textbf{Avazu}\footnote{https://www.kaggle.com/c/avazu-ctr-prediction/data} This dataset contains users' mobile behaviors, including whether a user clicks a displayed mobile ad or not. It has 23 feature fields spanning from user/device features to ad attributes. 3) \textbf{MovieLens-1M}\footnote{https://grouplens.org/datasets/movielens/} This dataset contains users’ ratings on movies. We treat samples with a rating of less than three as negative samples for binarization because a low score indicates that the user does not like the movie and will not click it for watching. We treat samples with a rating greater than three as positive samples. The statistics of the three datasets are summarized in \tablename\ \ref{tab:datasets}.

\begin{table}
\vspace{-0.2cm}
\caption{Statistics of evaluation datasets.}
\vspace{-0.2cm}
\label{tab:datasets}
\resizebox{\linewidth}{!}{%
\begin{tabular}{@{}cccc@{}}
\toprule
\multicolumn{1}{c}{Dataset} & \#Instances & \#Fields & \#Feature Dimension \\ \midrule
Criteo                      & 45,840,617  & 39       & 1,086,810           \\
Avazu                       & 14,426,917  & 23       & 767,250             \\
MovieLens-1M                & 1,000,209   & 5        & 10,072              \\ \bottomrule
\end{tabular}%
}
\vspace{-0.4cm}
\end{table}

\subsubsection{Data Preparation}
First, we set a frequency \textit{threshold} for removing all infrequent features and treat them as a single feature "$<$unknown$>$", where \textit{threshold} is set to $\{10, 4\}$ for Criteo, Avazu, respectively. Second, for numerical features with a high variance that might influence the learning process, we normalize them by transforming a value $z$ to $\log^{2}{z}$ if $z>2$, which is proposed by the winner of Criteo Competition\footnote{https://www.csie.ntu.edu.tw/~r01922136/kaggle-2014-criteo.pdf}. Third, we randomly select 80\% of all samples for training and consecutively split the rest into validation and test sets as 1:1.

\subsection{Evaluation Metrics}
We apply two famous metrics to evaluate the performance of our proposed approach and all other competitive models: \textbf{AUC} Area under ROC curve and \textbf{Logloss} (cross entropy). Note that \textbf{a slight increase in AUC or decrease in Logloss at 0.001-level} is known to be a significant improvement for the tasks such as CTR prediction \cite{Cheng2016,Guo2017,Wang2017}.

\subsection{Baselines}
We compare DCAP with existing methods of three classes (also see \tablename\ \ref{tab:performance}): (a) The first-order approaches that model a linear combination of raw features; (b) FM-based methods that capture second-order feature interactions; (c) Advanced approaches that model high-order cross features. 
\begin{itemize}
    \item \textbf{LR} logistic regression models the linear relationship between features and targets.
    \item \textbf{FM} \cite{Rendle2010} factorization machine applies matrix factorization techniques to capture the second order feature interactions for prediction.
    \item \textbf{AFM} \cite{He2017} attentional factorization machine assigns attention weights to feature interactions based on FM.
    \item \textbf{HOFM} \cite{Blondel2016} high-order factorization machine can model high-order feature interactions.
    \item \textbf{NFM} \cite{He2017a} neural factorization machine sums up pairwise Hadamard product of features followed by a fully connected neural network. 
    \item \textbf{PNN} \cite{Qu2017} product neural network models high-order feature interactions by computing pairwise inner (iPNN) or outer products (oPNN) of input features followed by a fully connected neural network. 
    \item \textbf{Wide \& Deep} \cite{Cheng2016} wide \& deep model integrates LR and DNN. Also, we omit the hand-crafted cross features for a fair comparison. 
    \item \textbf{DCN} \cite{Wang2017} deep \& cross network models the explicit feature, by computing the outer product between original input embedding and corresponding output across layers. 
    \item \textbf{DeepFM} \cite{Guo2017} deep factorization machine is an ensemble between DNN and FM. 
    \item \textbf{xDeepFM} \cite{Lian2018} xDeepFM models both explicit and implicit high-order feature interactions by computing outer products of feature vectors at different orders. It also combines with DNN. 
    \item \textbf{AutoInt} \cite{Song2019} auto feature interaction model applies multi-head self-attention mechanism on modeling cross feature interactions. 
    \item \textbf{AFN} \cite{Cheng2019} adaptive factorization network learns arbitrary-order cross features adaptively by building a logarithmic transformation layer that converts each feature's power in a feature combination into the coefficient to be learned. 
\end{itemize}

\subsection{Training Settings}
We implement our approaches using PyTorch\footnote{https://pytorch.org/}. We apply Adam \cite{Kingma2015} with a learning rate of 0.001 and a weight decay of $1e^{-6}$ to prevent overfitting, and a mini-batch size of 4096 across all tasks. For the fair competition, we set the default architecture of the dense neural network with two hidden layers and 100 neurons per layer for all models that involve DNN. To avoid overfitting, we perform the early-stopping strategy based on AUC on the validation set. A dropout method is also applied across all models with a rate of 0.5 for the MovieLens-1M dataset and 0.2 for the other two datasets to prevent overfitting. The dimension of feature embeddings is set to 16 for all the models across all tasks consistently. More specifically, the number of layers in DCN set to 2. The maximum order in HOFM is set to 3. The attention embedding size of model AFM and AutoInt is 64. Additionally, the number of heads in AutoInt is set to 4. The default number of logarithmic neurons in AFN is set to 1500, 1200, 800 for Criteo, Avazu, and Movielens datasets. For our model DCAP, we set the maximum depth of network to 2 as a bounded order of feature interactions. 

All the hyper-parameters are tuned on the validation set. We run the experiments for each empirical result as 20 independent trials on four NVIDIA Tesla P100 GPUs in parallel and report the average value with the standard deviation. 

\begin{table*}[!htb]
\centering
\vspace{-0.4cm}
\caption{The overall performance of all models on Criteo, Avazu, and MovieLens-1M datasets. Best performance in boldface. We also mark the second-best model with an underline. We further analyze these results in Section \ref{performance}.}
\label{tab:performance}
\vspace{-0.2cm}
\resizebox{\textwidth}{!}{%
\begin{tabular}{@{}llllllll@{}}
\toprule
\multirow{2}{*}{Model Class} &
  \multirow{2}{*}{Model} &
  \multicolumn{2}{c}{Criteo} &
  \multicolumn{2}{c}{Avazu} &
  \multicolumn{2}{c}{MovieLens-1M} \\
\multicolumn{1}{c}{} &
  \multicolumn{1}{c}{} &
  \multicolumn{1}{c}{AUC} &
  \multicolumn{1}{c}{Logloss} &
  \multicolumn{1}{c}{AUC} &
  \multicolumn{1}{c}{Logloss} &
  \multicolumn{1}{c}{AUC} &
  \multicolumn{1}{c}{Logloss} \\ \midrule
First-Order &
  LR &
  0.7943+/-0.0000 &
  0.4560+/-0.0000 &
  0.7608+/-0.0000 &
  0.3916+/-0.0000 &
  0.7918+/-0.0000 &
  0.5406+/-0.0000 \\ \midrule
\multirow{2}{*}{Second-Order} &
  FM &
  0.8040+/-0.0001 &
  0.4478+/-0.0001 &
  0.7816+/-0.0004 &
  0.3826+/-0.0005 &
  0.8004+/-0.0005 &
  0.5391+/-0.0017 \\
 &
  AFM &
  0.8073+/-0.0001 &
  0.4443+/-0.0001 &
  0.7756+/-0.0008 &
  0.3859+/-0.0017 &
  0.7983+/-0.0018 &
  0.5368+/-0.0021 \\ \midrule
\multirow{11}{*}{High-Order} &
  HOFM &
  0.8059+/-0.0001 &
  0.4459+/-0.0001 &
  0.7824+/-0.0002 &
  0.3820+/-0.0006 &
  0.7953+/-0.0008 &
  0.5381+/-0.0011 \\
 &
  NFM &
  0.8061+/-0.0006 &
  0.4455+/-0.0006 &
  0.7809+/-0.0010 &
  0.3892+/-0.0016 &
  0.7896+/-0.0175 &
  0.5512+/-0.0149 \\
 &
  iPNN &
  0.8118+/-0.0001 &
  0.4400+/-0.0001 &
  0.7847+/-0.0007 &
  0.3772+/-0.0007 &
  0.8039+/-0.0009 &
  0.5292+/-0.0012 \\
 &
  oPNN &
  {\underline{0.8132+/-0.0001} } &
  {\underline{0.4387+/-0.0001} } &
  {\underline{0.7851+/-0.0003} } &
  {\underline{0.3766+/-0.0001} } &
  0.8046+/-0.0008 &
  0.5289+/-0.0019 \\
 &
  Wide\&Deep &
  0.8107+/-0.0001 &
  0.4411+/-0.0001 &
  0.7732+/-0.0006 &
  0.3855+/-0.0003 &
  0.7999+/-0.0012 &
  0.5374+/-0.0009 \\
 &
  DCN &
  0.8112+/-0.0001 &
  0.4408+/-0.0001 &
  0.7833+/-0.0004 &
  0.3784+/-0.0002 &
  {\underline{0.8051+/-0.0010} } &
  {\underline{0.5271+/-0.0015} } \\
 &
  DeepFM &
  0.8086+/-0.0001 &
  0.4432+/-0.0000 &
  0.7757+/-0.0005 &
  0.3853+/-0.0005 &
  0.8008+/-0.0011 &
  0.5366+/-0.0017 \\
 &
  xDeepFM &
  0.8108+/-0.0002 &
  0.4411+/-0.0002 &
  0.7767+/-0.0007 &
  0.3845+/-0.0004 &
  0.8045+/-0.0009 &
  0.5297+/-0.0013 \\
 &
  AutoInt &
  0.8108+/-0.0001 &
  0.4411+/-0.0001 &
  0.7760+/-0.0003 &
  0.3846+/-0.0004 &
  0.8047+/-0.0006 &
  0.5289+/-0.0007 \\
 &
  AFN &
  0.8122+/-0.0000 &
  0.4397+/-0.0001 &
  0.7743+/-0.0012 &
  0.3876+/-0.0010 &
  0.7947+/-0.0021 &
  0.5424+/-0.0025 \\
 &
  \textbf{DCAP(ours)} &
  \textbf{0.8142+/-0.0001} &
  \textbf{0.4376+/-0.0001} &
  \textbf{0.7861+/-0.0003} &
  \textbf{0.3754+/-0.0002} &
  \textbf{0.8066+/-0.0012} &
  \textbf{0.5260+/-0.0013} \\ \bottomrule
\end{tabular}%
}
\vspace{-0.4cm}
\end{table*}

\subsection{Comparative Performance (\textbf{RQ1})}
\label{performance}

First, we want to know how our proposed model performs compared with various models involving first-order, second-order, and high-order feature interactions. Note that FM only models second-order cross features explicitly while DNNs model high-order feature interactions implicitly. Models integrated with DNNs such as NFM, PNN, DeepFM, Wide\&Deep, etc., are also capable of modeling high-order interactions. Additionally, models like DCN, xDeepFM, AutoInt, and AFN are all ensemble models that combine the effectiveness of jointly explicit and implicit learning, which means they can model high-order feature interactions both explicitly and implicitly. Our proposed DCAP is also such an ensemble model that cross attentional product explicitly models weighted high-order feature interactions followed by an implicit feed-forward neural network. 

The results shown in \tablename\ \ref{tab:performance} demonstrate that our method outperforms the other models across different datasets consistently. For the Criteo dataset, DCAP outperforms the second-best model oPNN by a significant 0.001 increase in AUC and 0.001 decreases in Logloss. It also has an 0.001 prediction performance increase in both AUC and Logloss on the Avazu dataset. For MovieLens-1M, the gap between DCAP and the second-best model is further widened. Compared to FM based models, DCAP has an almost 0.006 higher AUC score than DeepFM, which is generally considered a significant user response prediction benchmark. It also has a 0.0034 higher AUC score than xDeepFM, which also models high-order feature interactions explicitly yet without differentiating each cross feature's importance. Another interesting finding is that there is no theoretic guarantee of the superiority of high-order ensemble methods over the simple individual models since AFN performs worse than FM and AFM on both Avazu and MovieLens-1M datasets. Compared to the Criteo dataset, these two datasets contain fewer training instances and feature fields requiring minimal high-order feature interactions to make the predictions. As the AFN model learns arbitrary-order cross features, it can be hindered by overfitting the data due to the over-complicated feature interaction space. In this case, a simple model such as FM might achieve relatively better performance. Note that for the Avazu dataset, models such as DeepFM, xDeepFM, AutoInt also have worse performance than individual models like FM. These models share a typical architecture that combines the explicit cross features and implicit embeddings in an \textit{add} fashion. However, for PNN and DCN, they stack explicit feature interaction outputs and input embeddings together, followed by the DNNs. As such, this concatenated way of combining outputs may have a superiority than sum up. 

Generally, models integrated with DNNs have a better performance over individuals, which further illustrates the effectiveness of implicit feature interactions. Among all the integrated models, the results of PNN and DCN are more impressive due to the consistent outstanding prediction performance on three datasets. We analyze the potential effectiveness of these models and compare them with our method in the following. 

\subsubsection{Comparison Between DCAP and DCN}
Both DCAP and DCN adopt the \textit{cross network} architecture to explicitly model high-order cross features. However, the differences are: (a) DCN can learn feature interactions very efficiently (with much less computation), but the output of each cross network layer is limited to a scalar multiple of original input embedding \cite{Lian2018}, which may lead to a suboptimal model performance; (b) DCAP has a higher requirement for computational resources while can learn more elaborate cross feature interactions with discriminated attention or weights; (c) all interactions in DCN come at a bit-wise fashion, which brings more implicitness; however, DCAP models high-feature interactions at a vector-wise level which can be followed by an implicit bit-wise interaction layer. Though they are both jointly explicit and implicit learners, DCAP owns better interpretability. 

\subsubsection{Comparison Between DCAP and PNN}
The backbone of modeling cross features is inspired by PNN, which first applies inner product or outer product operations on modeling feature interactions. The downsides of Product Neural Network are: (a) PNN only models second-order explicit cross features while DNN implicitly captures the high-order cross features; DCAP can model high-order feature interactions both explicitly and implicitly; (b) similar to many other models, PNN is not capable of differentiating the importance of feature interactions. By deepening the cross product network combining with attention mechanism, we generalize the product network's capability to model high-order feature interactions explicitly. 

\begin{table}
\caption{AUC scores of model with different number of layers and heads on MovieLens-1M dataset.}
\label{tab:layer_vs_head_auc}
\vspace{-0.2cm}
\resizebox{\linewidth}{!}{%
\begin{tabular}{@{}c|ccccc@{}}
\toprule
\multicolumn{1}{l|}{\backslashbox{\#Layers}{\#Heads}}  & 1      & 2      & 4      & 8      & 16     \\ \midrule
1                     & 0.8066 & 0.8073 & 0.8065 & 0.8072 & 0.8068 \\
2                     & 0.8074 & 0.8074 & 0.8073 & 0.8075 & 0.8061 \\
3                     & 0.8079 & 0.8068 & 0.8053 & 0.8072 & 0.8066 \\
4                     & 0.8070 & 0.8078 & 0.8073 & 0.8055 & \textbf{0.8081} \\
5                     & 0.8075 & 0.8071 & 0.8076 & 0.8078 & 0.8068 \\ \bottomrule
\end{tabular}%
}
\vspace{-0.5cm}
\end{table}

\subsection{Hyper-parameter Investigation (\textbf{RQ2})}
One intuitive question is that how the multi-head attention and depth of the network affect the model performance. In this section, we study hyper-parameters' impact, including (a) the number of heads in multi-head attention; and (b) the number of cross layers in the network. We conduct experiments via holding the identical settings for the DNN part while varying the cross attentional product network part settings. We show how the validation ACU and Logloss change by adding more layers and heads when setting the embedding size to 64. \tablename\ \ref{tab:layer_vs_head_auc} and \tablename\ \ref{tab:layer_vs_head_logloss} show the changes in model performance in terms of AUC and Logloss on the MovieLens dataset, respectively. 

\begin{table}
\caption{Logloss of model with different number of layers and heads on MovieLens-1M dataset.}
\label{tab:layer_vs_head_logloss}
\vspace{-0.2cm}
\resizebox{\linewidth}{!}{%
\begin{tabular}{@{}c|rrrrr@{}}
\toprule
\multicolumn{1}{l|}{\backslashbox{$\#$Layers}{$\#$Heads}} & \multicolumn{1}{c}{1} & \multicolumn{1}{c}{2} & \multicolumn{1}{c}{4} & \multicolumn{1}{c}{8} & \multicolumn{1}{c}{16} \\ \midrule
1 & 0.5297 & 0.5286 & 0.5293 & 0.5296 & 0.5274 \\
2 & 0.5282 & 0.5278 & 0.5282 & 0.5300 & 0.5303 \\
3 & 0.5265 & 0.5296 & 0.5290 & 0.5270 & 0.5302 \\
4 & 0.5278 & 0.5294 & 0.5276 & 0.5294 & 0.5266 \\
5 & 0.5295 & 0.5278 & 0.5275 & 0.5266 & \textbf{0.5265} \\ \bottomrule
\end{tabular}%
}
\vspace{-0.5cm}
\end{table}

\subsubsection{Number of Heads}
The number of heads is a critical hyper-parameter in multi-head attention. With each attention \textit{head} potentially focusing on different parts of the input and revealing complementary information, this attention mechanism can express sophisticated functions beyond the simple weighted average. However, recent research \cite{Michel2019} has discovered that not all heads are practically necessary since only a small percentage of heads will significantly impact the performance. Besides, many heads will have a higher chance of modeling a more sophisticated combination of input embedding and overfitting the data. We then hypothesize that the number of heads is not the more, the better. As such, it is worth investigating how the number of heads affects model performance in practice. Besides the experiment results listed in \tablename\ \ref{tab:layer_vs_head_auc} and \tablename\ \ref{tab:layer_vs_head_logloss}, we further visualize the model performance against the number of heads in \figurename\ \ref{fig:avazu_heads} and \figurename\ \ref{fig:movie_heads} fixing other chosen randomly hyper-parameters. The performance generally ascends first and then descends as the number of heads increases. Specifically, when we have sixteen heads for the model, the performance drops down heavily, further demonstrating our conjecture. The model with eight heads has an astonishing AUC score, while the Logloss is not good on MovieLens dataset. For Avazu dataset, the model can achieve the best performance with a total of four heads. In all, we conclude that the model can achieve relatively good performance with at most four heads on both Avazu and MovieLens dataset.

\begin{figure}[htb]
     \centering
     \begin{subfigure}[b]{0.45\linewidth}
         \centering
         \includegraphics[width=\textwidth]{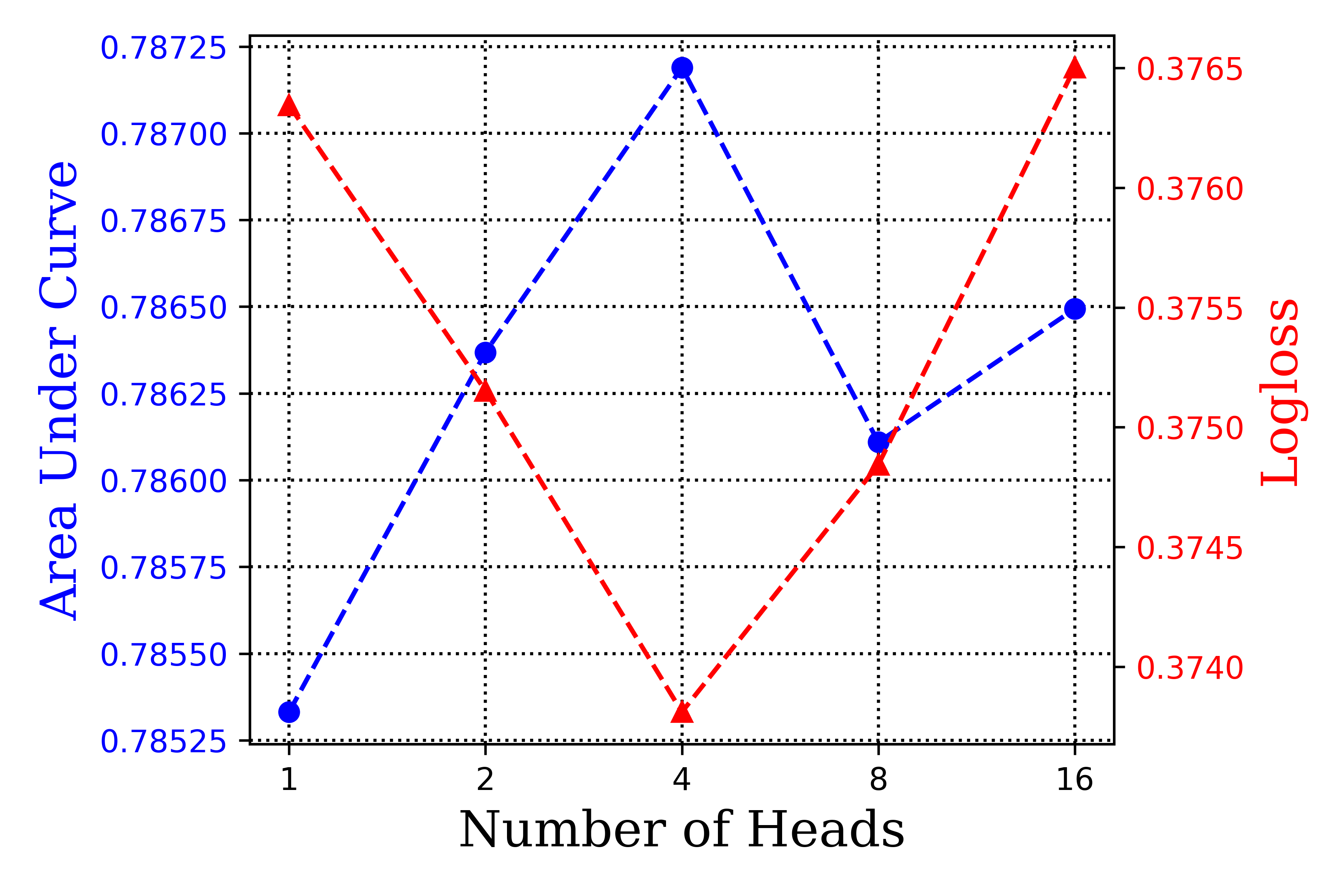}
         \caption{Model performance against the number of heads.}
         \label{fig:avazu_heads}
     \end{subfigure}
     \hfill
     \begin{subfigure}[b]{0.45\linewidth}
         \centering
         \includegraphics[width=\textwidth]{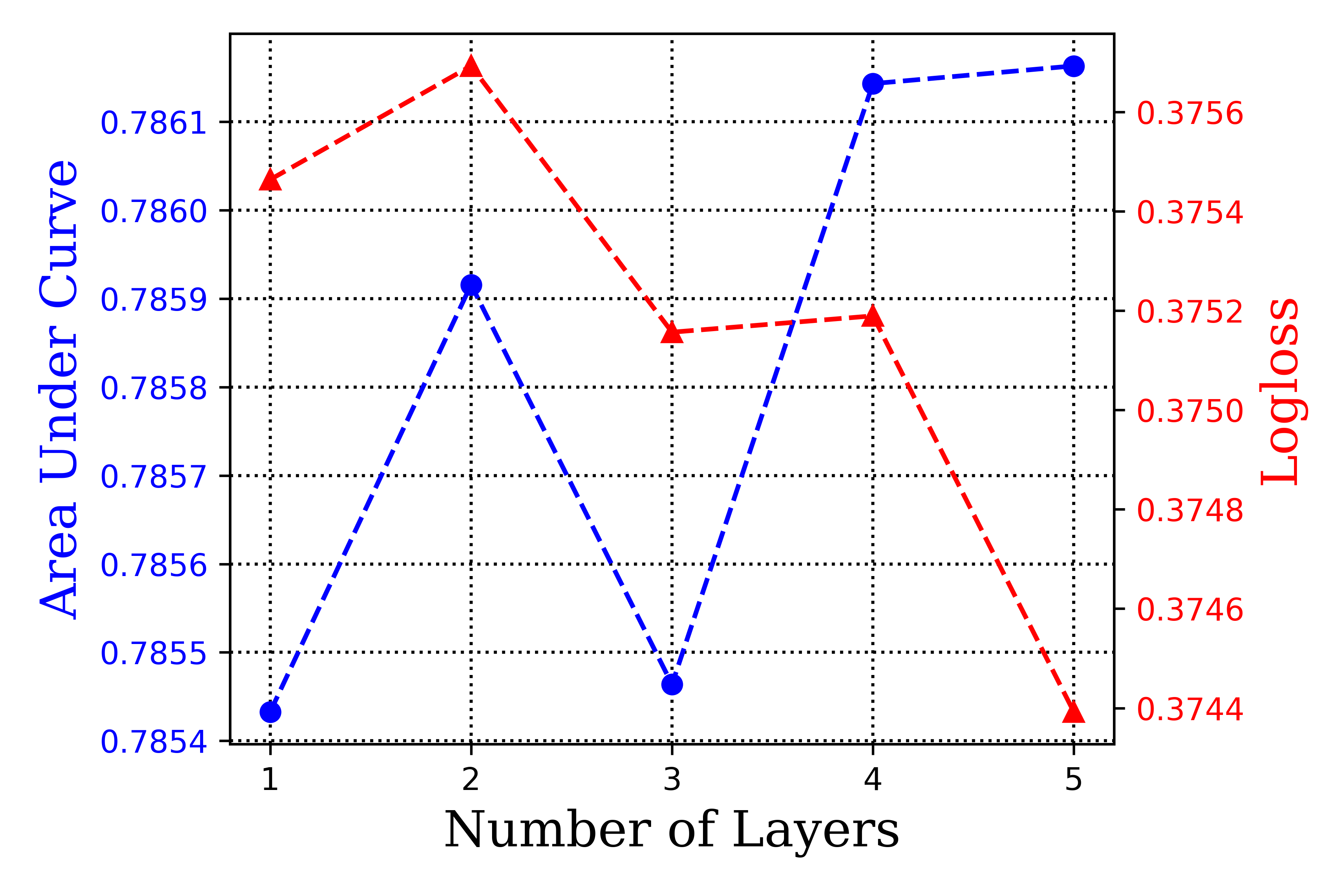}
         \caption{Model performance against the depth of network.}
         \label{fig:avazu_layers}
     \end{subfigure}
        \vspace{-0.2cm}
        \caption{Results on the Avazu dataset.}
        \label{fig:avazu}
        \vspace{-0.2cm}
\end{figure}

\begin{figure}[htb]
     \centering
     \vspace{-0.2cm}
     \begin{subfigure}[b]{0.45\linewidth}
         \centering
         \includegraphics[width=\textwidth]{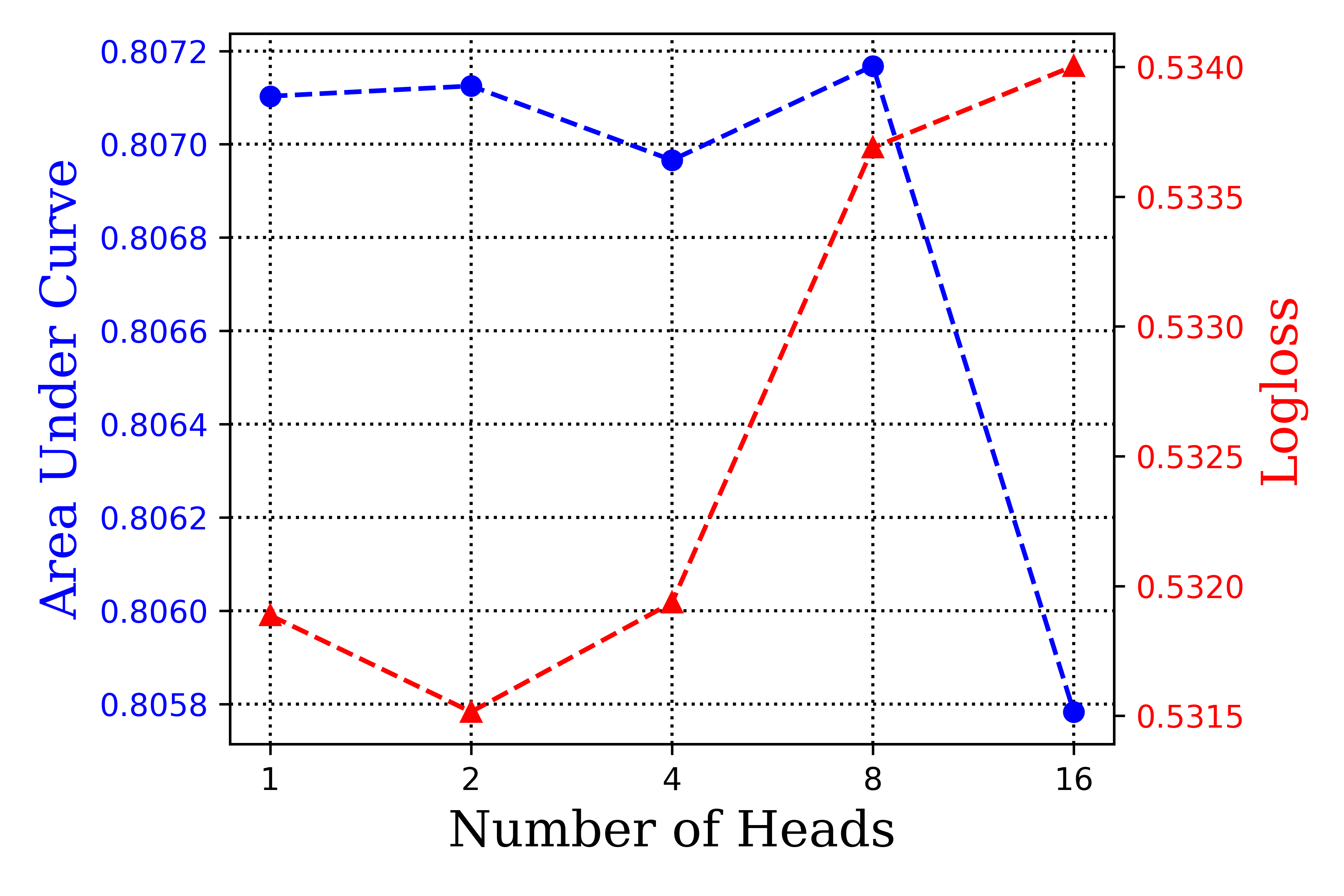}
         \caption{Model performance against the number of heads.}
         \label{fig:movie_heads}
     \end{subfigure}
     \hfill
     \begin{subfigure}[b]{0.45\linewidth}
         \centering
         \includegraphics[width=\textwidth]{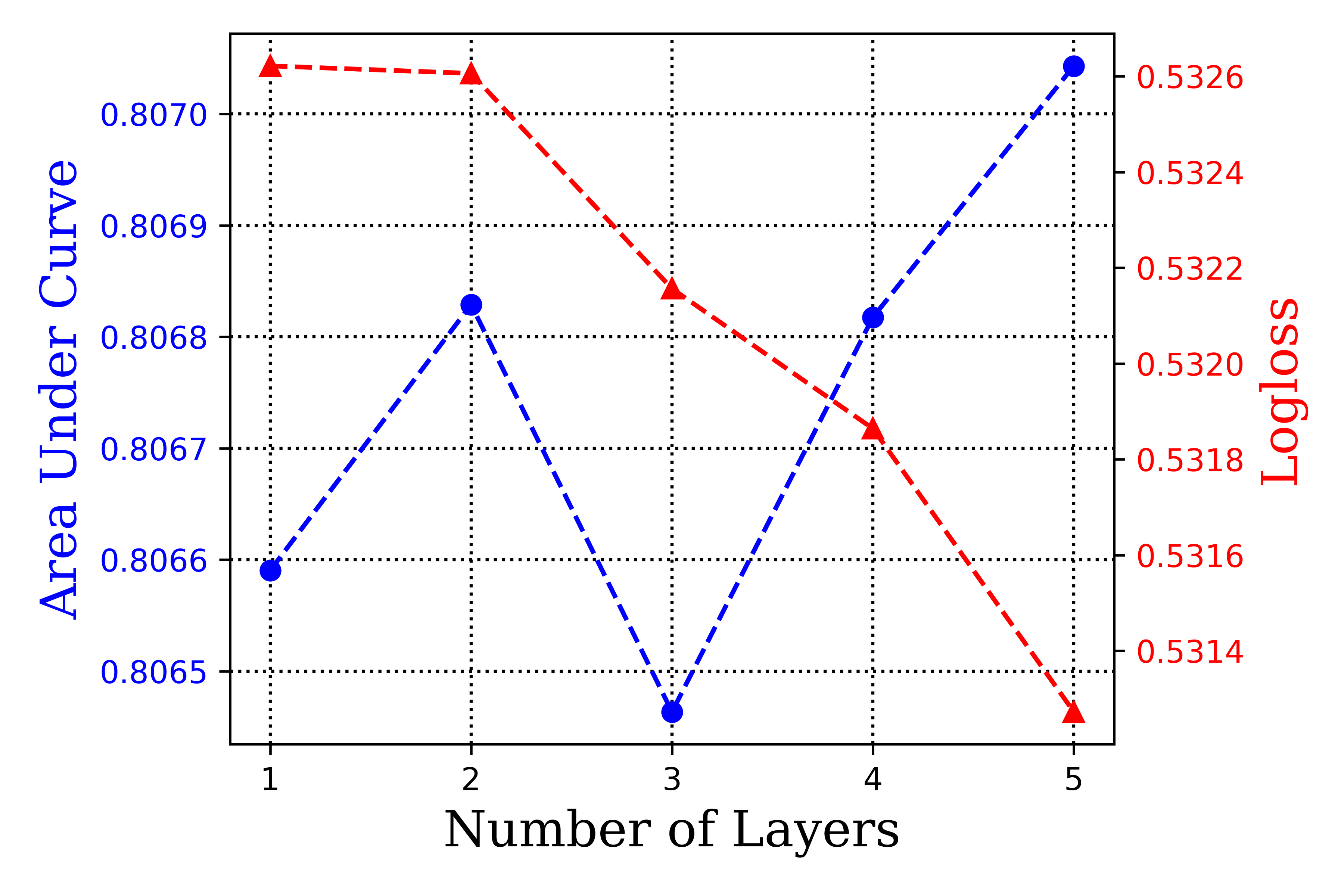}
         \caption{Model performance against the depth of network.}
         \label{fig:movie_layers}
     \end{subfigure}
        \vspace{-0.2cm}
        \caption{Results on the MovieLens dataset.}
        \label{fig:movielens}
        \vspace{-0.2cm}
\end{figure}

\subsubsection{Depth of Network}
The depth of network controls the upper bound degree of cross features. The deeper the network is, the higher order of feature interactions are captured by the model. Meanwhile, a deeper network involves more computational complexity which may cause overfitting. To investigate the impact of number of layers, we increase the depth from 1 to 5 and report the AUC and Logloss results on Avazu and MovieLens datasets. We also visualize the model performance against number of layers in \figurename\ \ref{fig:avazu_layers} and \figurename\ \ref{fig:movie_layers}. We observe that as the depth of network increases, the performance generally increases since the model captures more sophisticated high-order feature interactions. However, model complexity is linearly related to the number of layers and it could be a trade-off between prediction precision and computational requirements. 

\begin{figure}[htbp]
     \centering
     \begin{subfigure}[b]{0.24\linewidth}
         \centering
         \includegraphics[width=\textwidth]{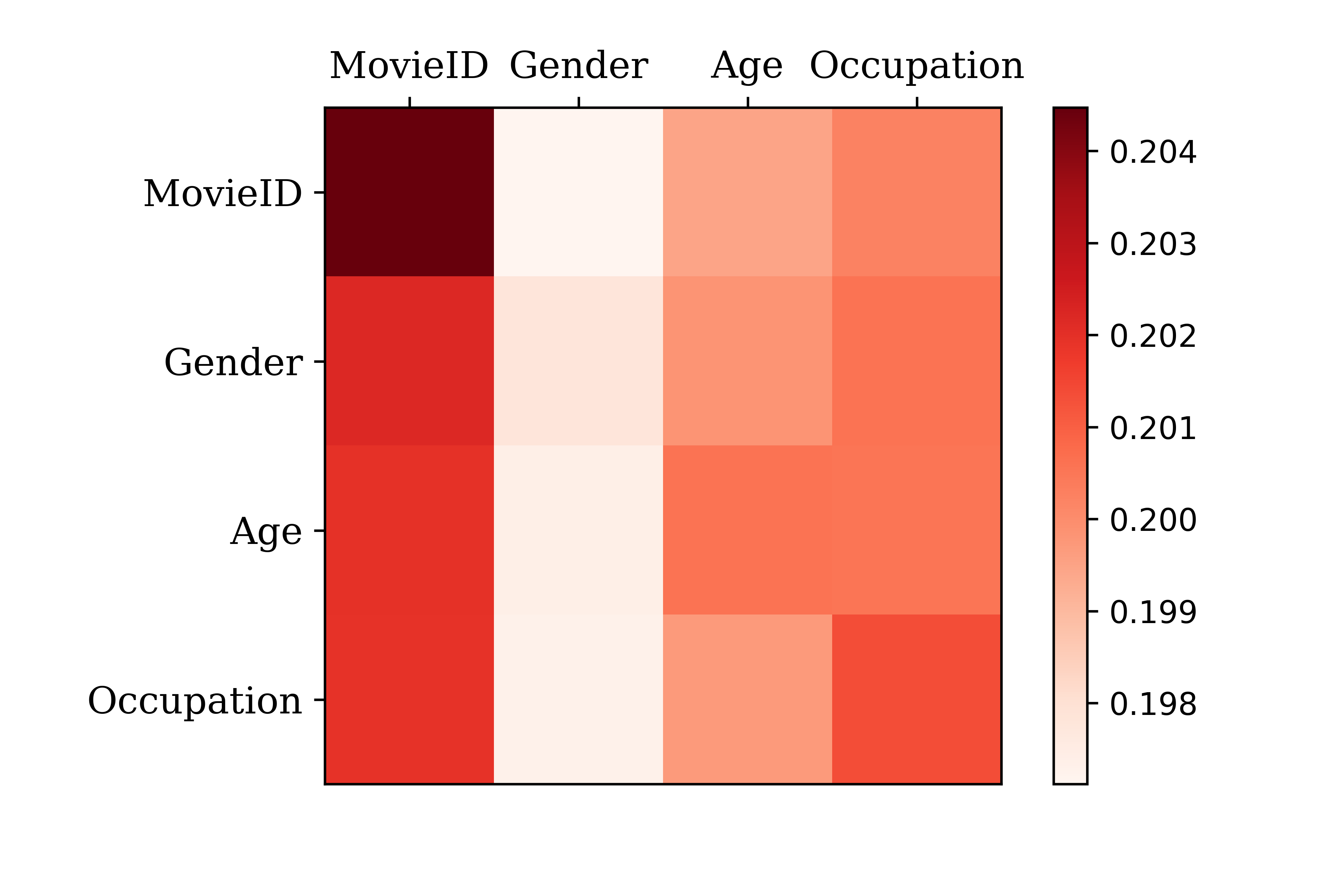}
         \label{fig:head1}
     \end{subfigure}
     \hfill
     \begin{subfigure}[b]{0.24\linewidth}
         \centering
         \includegraphics[width=\textwidth]{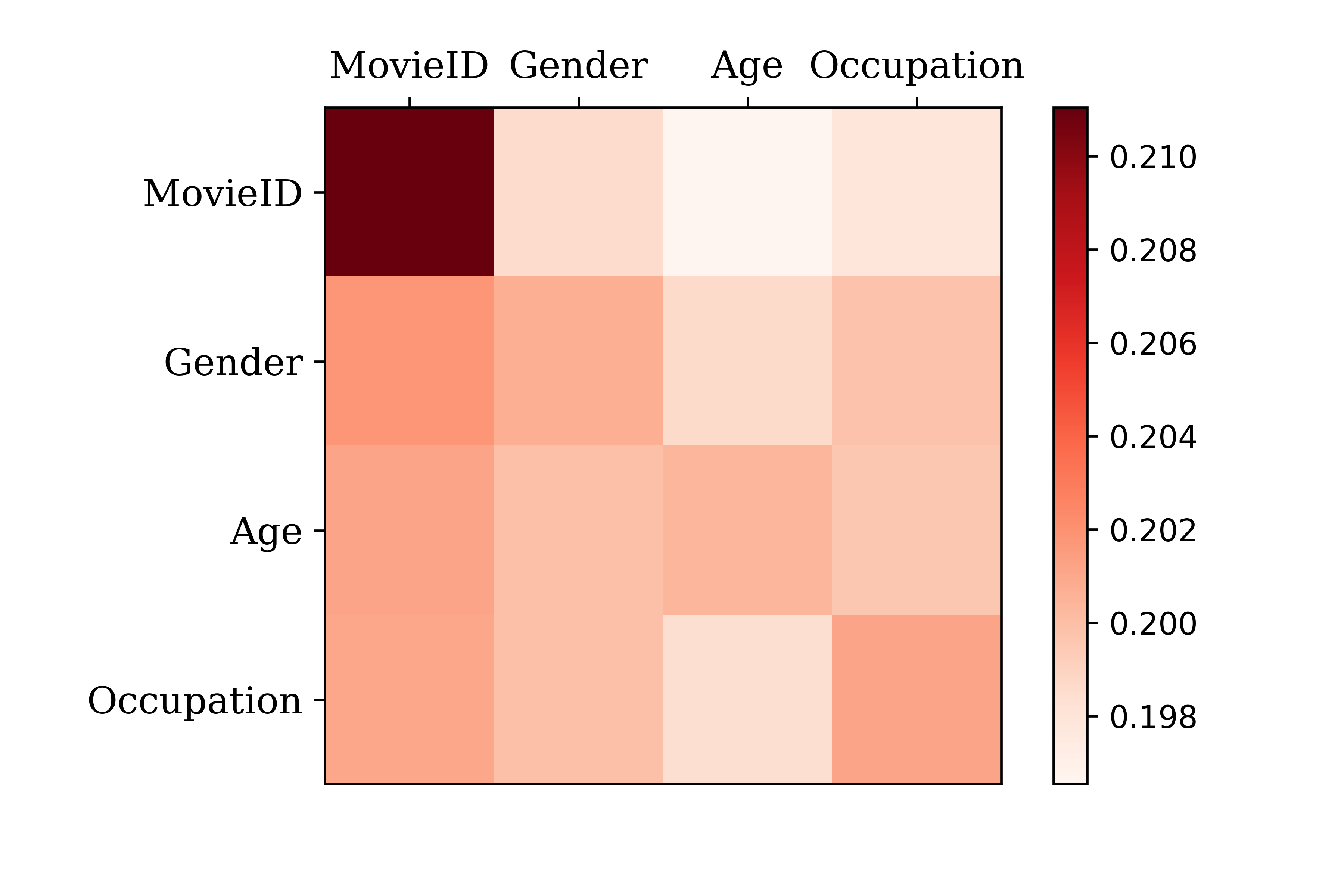}
         \label{fig:head2}
     \end{subfigure}
     \hfill
     \begin{subfigure}[b]{0.24\linewidth}
         \centering
         \includegraphics[width=\textwidth]{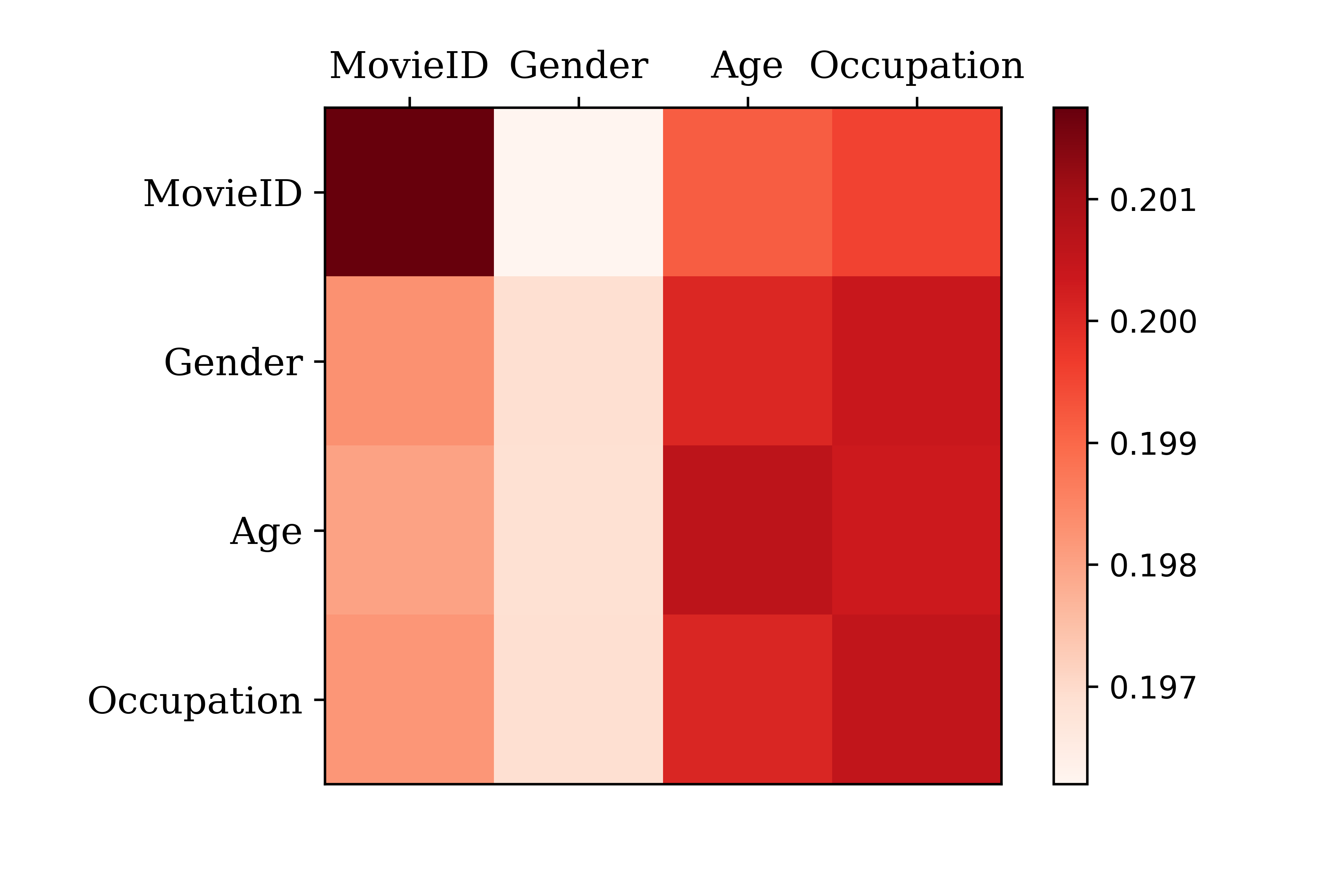}
         \label{fig:head3}
     \end{subfigure}
     \hfill
     \begin{subfigure}[b]{0.24\linewidth}
         \centering
         \includegraphics[width=\textwidth]{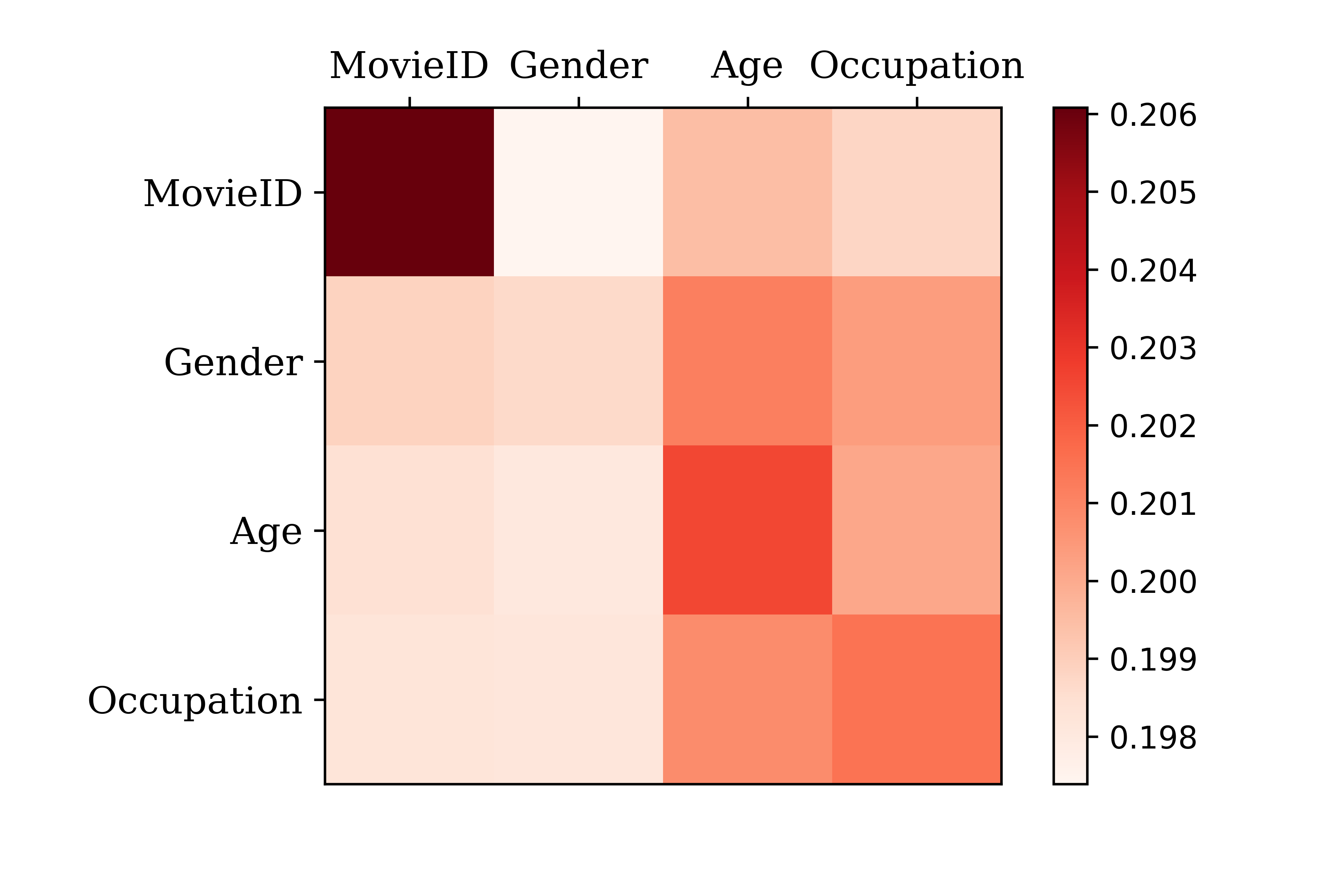}
         \label{fig:head4}
     \end{subfigure}
        \vspace{-0.5cm}
        \caption{The attention pattern of each head on MovieLens dataset.}
        \vspace{-0.6cm}
        \label{fig:multi-head attention}
\end{figure}

\subsection{Is Multi-head Attention Really Necessary? (\textbf{RQ3})}
Another straightforward question is that is multi-head attention crucial for our model? Unlike the multi-head attention mechanism, Attentional Factorization Machine (AFM) generally applies a DNN on the inner product of feature embeddings followed by a softmax function to compute the attention scores, allowing different cross features contribute differently to the prediction. 

Since multi-head attention is more computationally expensive, we further investigate whether this multi-head attention mechanism is necessary for user response prediction tasks. We conduct experiments on replacing the multi-head attention with this less complicated attention mechanism proposed by AFM and compare it with our model. \tablename\ \ref{tab:without multihead} shows the results of comparison on two datasets, Avazu and MovieLens-1M. We let \textit{w/} represent the original multi-head attention mechanism used in our model, while \textit{w/o} stands for the attention mechanism used in AFM. As we can see, the original DCAP model consistently outperforms the model with the attention mechanism used in AFM, which further demonstrates the superiority of multi-head attention mechanism in modeling pairwise feature relevance. 

To better understand the importance of multi-head attention, we also investigate the patterns learned by each attention \textit{head}. We select out four feature fields \textit{MovieID}, \textit{Gender}, \textit{Age} and \textit{Occupation} from the MovieLens dataset and visualize the attention pattern of each attention \textit{head} (four in all). As \figurename\ \ref{fig:multi-head attention} shows, the color darkness represents the degree of correlation or attention weights between pairwise feature fields. We observe that different heads attend to different parts of the feature fields as the first \textit{head} concentrates more on the cross correlation between \textit{MovieID} and other fields (left half). In contrast, the third \textit{head} focuses on the cross dependencies between \textit{Age}, \textit{Occupation} and other features (right half). It is also obvious that \textit{Gender} is the least attentional feature in predicting whether a user will watch a certain movie, which is reasonable to some extent. 

The great interpretability of our model gives much more room for data engineers to analyze user behaviors further. More targeting strategies could be made based on these results to improve the user experience or expand the market.

\begin{table}
\caption{The performance comparison between multi-head attention and AFM attention mechanism.}
\label{tab:without multihead}
\vspace{-0.2cm}
\resizebox{\linewidth}{!}{%
\begin{tabular}{@{}llcc@{}}
\toprule
Datasets    & Models       & AUC & Losloss \\ \midrule
\multirow{3}{*}{Avazu}        & DCAP$_{w/}$  &   \textbf{0.7861+/-0.0003}  &   \textbf{0.3754+/-0.0002}      \\
             & DCAP$_{w/o}$ &   0.7830+/-0.0004  &    0.3818+/-0.0005     \\
\multirow{3}{*}{MovieLens-1M} & DCAP$_{w/}$  &  \textbf{0.8066+/-0.0012} &    \textbf{0.5260+/-0.0013}     \\
             & DCAP$_{w/o}$ &  0.8043+/-0.0015 & 0.5290+/-0.0015   \\ \bottomrule
\end{tabular}%
}
\vspace{-0.5cm}
\end{table}

\section{Conclusion}
\label{conclusion}
This paper proposed a novel model Deep Cross Attentional Product Network (DCAP), which can learn high-order feature interactions explicitly at the vector-wise level. It also had the capability of differentiating the importance of all cross features of any degree of order. The core of our models consisted of three parts: (a) multi-head attention for capturing pairwise feature relevance
Product of features across layers for modeling high-order cross features; (b)
concatenation of output from each layer for retaining information of all feature interactions at any order; (c) experimental results on three real-world datasets demonstrate the effectiveness and superiority of our model in predicting user response. Additionally, we demonstrated our model interpretability via visualizing the feature correlations of different attention heads. After incorporating DNNs for modeling implicit feature interactions, our model consistently achieved better AUC and Logloss scores than state-of-the-art models. 

For future work, we will further explore the other attention variants with less computation complexity. Recent work such as Reformer \cite{Kitaev2020} or Linformer \cite{Wang2020} have provided a faster version of multi-head attention with only $\mathcal{O}(n\sqrt{n})$ and $\mathcal{O}(n)$ complexity, respectively. Besides the attention blocks, we are also interested in redesigning computing feature production in a more efficient way such that the model can be trained on GPU clusters distributively.

\bibliographystyle{ACM-Reference-Format}
\bibliography{acmart}

\appendix

\end{document}